\documentclass{article} 
\usepackage{iclr2025_conference,times}
\usepackage{textcomp}
\usepackage{gensymb}
\definecolor{commentcolor}{RGB}{110,154,155}   
\usepackage{ulem}
\usepackage{graphics}
\usepackage{tabularx}
\usepackage{booktabs}
\usepackage{multicol}
\usepackage{multirow}
\usepackage{bbding}
\usepackage{caption}
\usepackage{diagbox}

\usepackage{adjustbox}
\usepackage{url}            
\usepackage{booktabs}       
\usepackage{amsfonts}       
\usepackage{nicefrac}       
\usepackage{microtype}      
\usepackage{amsmath}
\usepackage{amssymb}
\usepackage{gensymb}
\usepackage{amsthm}

\usepackage{xcolor}

\usepackage{colortbl}
\usepackage{wrapfig}
\usepackage{subcaption}
\definecolor{CiteBlue}{HTML}{2471a3}
\definecolor{Highlight}{HTML}{C74167}
\usepackage[colorlinks,
            linkcolor=red,       
            anchorcolor=CiteBlue,  
            citecolor=CiteBlue,        
            ]{hyperref}

\newsavebox{\measurebox}

\definecolor{mygray}{gray}{.8}
\definecolor{mypink}{rgb}{.99,.91,.95}
\definecolor{mycyan}{cmyk}{.3,0,0,0}
\newtheorem{proposition}{Proposition}
\newtheorem{lemma}{Lemma}
\newtheorem{corollary}{Corollary}

\usepackage{color, colortbl, xcolor}
\definecolor{Gray}{gray}{0.9}

\usepackage{hyperref}
\usepackage{url}

\usepackage[capitalize]{cleveref}
\title{Test-time adaptation for image compression with distribution regularization}

\iclrfinalcopy
\author{Kecheng Chen$^1$, Pingping Zhang$^2$, Tiexin Qin$^1$, Shiqi Wang$^2$, Hong Yan$^1$ \& Haoliang Li$^1$ \\
$^1$Department of Electrical Engineering and the Centre for Intelligent Multidimensional Data \\Analysis (CIMDA) , City University of Hong Kong, China \\
$^2$Department of Computer Science, City University of Hong Kong, China  \\
\texttt{\{cs.ckc96, ppingyes, tiexinqin\}@gmail.com} \\ \texttt{\{shiqwang, h.yan, haoliang.li\}@cityu.edu.hk} \\ 
}

\begin{document}
\maketitle

\begin{abstract}
 Current test- or compression-time adaptation image compression (TTA-IC) approaches, which leverage both latent and decoder refinements as a two-step adaptation scheme, have potentially enhanced the rate-distortion (R-D) performance of learned image compression models on cross-domain compression tasks, \textit{e.g.,} from natural to screen content images.  However, compared with the emergence of various decoder refinement variants, the latent refinement, as an inseparable ingredient, is barely
 tailored to cross-domain scenarios. To this end, we aim to develop an advanced latent refinement method by extending the effective hybrid latent refinement (HLR) method, which is designed for \textit{in-domain} inference improvement but shows noticeable degradation of the rate cost in \textit{cross-domain} tasks. Specifically, we first provide theoretical analyses, in a cue of marginalization approximation from in- to cross-domain scenarios,  to uncover that the vanilla HLR suffers from an underlying mismatch between refined Gaussian conditional and hyperprior distributions, leading to deteriorated joint probability approximation of marginal distribution with increased rate consumption. To remedy this issue, we introduce a simple Bayesian approximation-endowed \textit{distribution regularization} to encourage learning a better joint probability approximation in a plug-and-play manner. Extensive experiments on six in- and cross-domain datasets demonstrate that our proposed method not only improves the R-D performance compared with other latent refinement counterparts, but also can be flexibly integrated into existing TTA-IC methods with incremental benefits.

\end{abstract}

\section{Introduction}
With rapid developments in data streaming techniques, fruitful high-resolution images need to be transmitted online between edge devices. It is therefore imperative to develop more efficient, effective, and versatile image compression approaches for better storage and transmission. To this end, we have witnessed recent learning-based image compression (LIC) methods~\citep{balle2016end,balle2018variational,cheng2020learned,jiang2023mlic,kim2024c3} significantly outperform conventional codecs, such as VVC~\citep{bross2021overview} and JPEG~\citep{wallace1991jpeg}, in terms of rate-distortion (R-D) performance. Such gains mainly derive from the unprecedented non-linear transform capacity of deep neural networks (DNN) and accurate probability representations for entropy coding, in an end-to-end R-D cost-guided learning framework.

Nevertheless, these DNN-based LIC approaches inevitably inherit the identically and independently distributed (\textit{i.i.d.}) assumption between the source (training data) and target (testing data) domains, which may not always hold in versatile image compression scenarios, \textit{e.g.,} there is a significant distribution gap between natural and screen content images. Compared with in-domain compression (\textit{i.e.,} \textit{i.i.d.} assumption holds), such \textit{domain shifts} would deteriorate the R-D performance of DNN-based codecs in cross-domain compression. 
For example, most advanced LIC models leverage a hyperprior-based entropy framework~\citep{balle2018variational}, where the hyperprior model extracts the side information $z$ of the latent variable $y$ to capture Gaussian conditional probability $p(y|z)$ for coding, and the side information is coding by the hyperprior probability $p(z)$ learned from entropy bottleneck. In cross-domain scenarios, the discrepancy in statistical property between source and target domains will cause inaccuracy or ineffectiveness of learned probability models~\citep{ulhaq2024encodingdistributions}, leading to suboptimal entropy coding with additional bit consumption. Moreover, it is difficult for a source-trained decoder to render high-fidelity target images due to domain shifts.

%

\setlength\intextsep{-3pt}
\begin{wrapfigure}[15]{r}{0cm}
	\centering
	\includegraphics[width=0.37\textwidth]{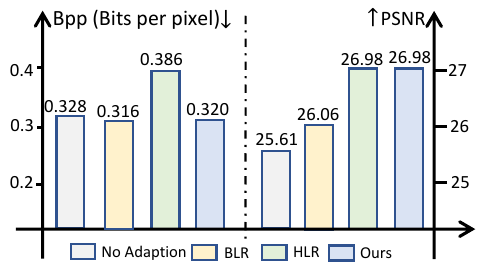}
	\caption{Comparison of various latent refinement methods in R-D performance under cross-domain tasks (testing on SIQAD screen content dataset). The Cheng20~\citep{cheng2020learned} (quality$=0$, pre-trained on natural images) model is used.}
	\label{quality 0}
\vspace{5pt}
\end{wrapfigure}
To this end, it is necessary to conduct a test- or compression-time adaptation for image compression (TTA-IC). One approach~\citep{van2021overfitting,lam2020efficient,zou2021adaptation} is to update the decoder of the entropy model, which involves the transmission of adapted parameters to the decoder side. Although recent parameter-efficient transfer learning (PETL)-based~\citep{hu2021lora} extensions~\citep{shen2023dec,lv2023dynamic} reduce the huge bitrate overhead of parameter transmission to an acceptable level, the entropy optimization of updated parameters and extra bit consumption are still troublesome. Another promising approach~\citep{djelouah2019content} is to directly refine the latent variables without altering any model parameters. Compared with the first approach, such a branch is still effective for many specialized neural decoders, where model parameters are hard-coded and non-modified~\citep{dass2023vitality}.

Here, we focus on the second approach, \textit{i.e.,} latent refinement, due to its simple optimization, friendly bitrate, and natural immunity to catastrophic forgetting. 
Some state-of-the-art (SOTA) TTA-IC methods~\citep{campos2019content,shen2023dec}  actually have introduced basic latent refinement (BLR)~\citep{djelouah2019content} into their two-step adaptation framework including latent and decoder refinements, \textit{e.g.,} updating the latent variable $y$ by R-D objective. Yet, the Gaussian conditional probability $p(y|z)$ and hyperprior $p(z)$ are still inaccurate without refining the side information $z$, leading to suboptimal cross-domain R-D performance. A potential solution may be hybrid latent refinement (HLR)~\citep{yang2020improving} that conducts a joint update of latent variable $y$ and side information $z$. 
Although the vanilla HLR, tailored to \textit{in-domain} inference improvement,  exhibits better R-D performance compared with BLR in \textit{in-domain} TTA-IC tasks~\citep{yang2020improving}, it suffers from a significant downside in rate cost while enhancing the reconstruction quality in \textit{cross-domain} scenarios, as depicted in Figure \ref{quality 0}. This degradation in rate cost is, however, ignored by existing two-step TTA-IC approaches~\citep{lv2023dynamic,tsubota2023universal,shen2023dec} that directly impose the vanilla HLR as their inseparable ingredient.

Motivated by the abovementioned analyses, this study aims to develop an advanced latent refinement method that can be adaptive to cross-domain TTA-IC with \textit{consistent} R-D gains. Such a scheme can not only render better latent representations with \textit{zero} model update and transmission but also enhance the R-D performance of existing SOTA TTA-IC approaches as an effective alternative to existing latent refinement. With these goals in mind, we propose to tailor the vanilla HLR method designed for \textit{in-domain} inference improvement to \textit{cross-domain} cases. To achieve this, a theoretical analysis, in a cue of marginalization approximation from in- to cross-domain scenarios, is provided to reveal the degradation reasons of the vanilla HLR in the cross-domain scenario. In a nutshell, we uncover that the underlying mismatch between refined Gaussian conditional and hyperprior distributions may trigger the deteriorated joint probability approximation of marginal distribution, leading to increased rate consumption. To remedy this issue, we introduce a novel \textit{distribution regularization} to the existing R-D objective, which encourages learning a better joint probability approximation from a theoretical perspective. Moreover, we impose a Bayesian approximation of the proposed distribution regularization to circumvent any model modification in a plug-and-play manner. Experiments on in- and cross-domain tasks demonstrate that our proposed method surpasses other latent refinements approaches and contribute to SOTA TTA-IC models with increased gains.

\section{Related works}

To mitigate the domain gap issue, various studies have been conducted for test- or compression-time adaptation of image compression. One strategy entails fine-tuning the encoder during inference to refine the latent variables without introducing extra bit transmission. For instance, \citet{djelouah2019content} addressed this challenge through the BLR, which employs gradient descent on the latent representation under the guidance of R-D cost. \citet{guo2020variable} proposed a two-step BLR extension that first optimizes the latent variable by R-D cost and further optimizes the side information by the rate cost. Beyond the BLR,  \citet{yang2020improving} proposed the hybrid latent refinement (HLR), which not only conducts a joint latent refinement for latent variable and side information but also eliminates the discretization gap of latent representations. Although extensive experiments \citep{yang2020improving} demonstrated that 
the HLR-based method can significantly enhance the R-D performance compared with BLR-based counterparts for \textit{in-domain} inference improvements, its performance in \textit{cross-domain} scenarios with strong domain shifts lacks comprehensive understanding.

Besides the latent refinement, decoder refinement is a widely used test-time adaptation method for image compression, where the model parameters on the decoder side are updated and transmitted. For example, \citet{van2021overfitting} performed comprehensive parameter updates in the decoder and entropy model using rate-distortion optimization.  \citet{lam2020efficient} solely updated the biases in convolutional layers within the post-processing network. \citet{zou2021adaptation} introduced multiplicative parameters that can overfit and update them for intraframe coding. Despite achievements, these approaches suffer from relatively more bits when compressing a single image due to the updates of many model parameters~\citep{kim2024c3}. To address this issue, 
very recent approaches motivated by parameter-efficient transfer learning (PETL)~\citep{hu2021lora} aimed to introduce the adaptor of the decoder for better reconstruction quality with lower bit consumption~\citep{shen2023dec,tsubota2023universal}. For example, \citet{lv2023dynamic} proposed to make a marriage between latent refinement and low-rank adaptation of the decoder, thereby achieving SOTA performance on cross-domain tasks.

In a nutshell, recent SOTA TTA-IC approaches \citep{shen2023dec,tsubota2023universal,lv2023dynamic} mainly leverage a combination of latent refinement and PETL-based decoder refinement. However, compared with the emergence of various
PETL-based decoder refinement variants, the latent refinement is barely tailored to cross-domain scenarios, thereby exhibiting underexplored space.  


\section{Methodology}
We first explain existing learned image compression from a perspective of marginalization approximation in the context of in-domain R-D cost-guided training. Then, we analyze the practical marginalization approximation when applying the learned image codec to cross-domain image compression. We further render a simple solution to achieve better R-D performance for cross-domain image compression based on the analysis of practical marginalization approximation.

\subsection{In-domain Marginalization Approximation}


First, we introduce the marginalization approximation used by existing entropy models~\citep{balle2018variational,cheng2020learned,zou2022devil}. Specifically, the true (unknown) marginal distribution $p(y)$ of the latent variable $y$ is usually approximated by observed data $y$ conditioned 
on its hyper latent variable $z$ (\textit{a.k.a.}, side information), which makes the conditional distribution $p(y|z)$ tractable, where hyperprior models are adopted to render such hyper latent variable $z$. In the context of hyperprior-based entropy models, the following result holds,
 \begin{lemma}[\citet{balle2018variational}]
 \label{lemma1}
    Let  $p(y|z)$ and $p(z)$ be accessed, a joint probability over $y$ and $z$ can be constructed to approximate the true marginal probability over $y$,
    \begin{equation}
        p(y) = \int p(y,z) dz,  \quad p(y,z) = p(y|z) \cdot p(z), \quad s.t., y = g_{a}(x), z = h_{a}(y),
    \end{equation}
    where $x$, $g_{a}(\cdot)$, and $h_{a}(\cdot)$ denote a raw image, an analysis transform function of the entropy model, and an analysis transform function of the hyperprior model.
 \end{lemma}Lemma \ref{lemma1} implies that the compression of the raw image involves the compression of $z$ using the learned prior distribution $p(z)$ and further compression of $y$ using the learned conditional distribution $p(y|z)$. Practically, we implement  $p(z)$ by a non-parametric, fully-factorized density model (\textit{a.k.a.}, entropy bottleneck)~\citep{balle2016end}, and implement $p(y|z)$ by a mean-scale Gaussian  model, \textit{i.e.,}
\begin{equation}
    p(z) = p(z|\mathbf{\varphi}), p(y|z) = p(y|z,\theta_{h_{e}})=\mathcal{N}(\mu,\sigma), \mu,\sigma = h_{e}(z;\theta_{h_{e}}), \label{distribution formulate}
\end{equation} where $\varphi$ denotes the learnable parameters of the density model characterized by the univariate distribution, and $h_{e}(\cdot)$ denotes a synthesis transform function of the hyperprior model parameterized by learnable $\theta_{h_{e}}$. The learning of these parameterized probability models derives from training data, which means that the \textit{i.i.d.} assumption between training and testing data is implicitly needed. We therefore call such marginalization approximation as the \textit{in-domain} one.

The widely used R-D objective can also be formulated  by maximizing the probability of joint distribution $p(y,z)$ and the posterior distribution $p(x|y,\theta_{g_{e}})$, \textit{i.e.,}
\begin{eqnarray}
    \mathcal{L}_{rd} = -\log [ p(y|z)\cdot p(z)] + \lambda (-\log p(x|y)) 
    = \underbrace{-\log p(y|z) - \log p(z)}_{Rate} + \underbrace{\lambda \|x - g_{e}(y)\|_{2}^{2}}_{Distortion}, \label{rd-in-domain}
\end{eqnarray}
where $g_{e}(\cdot)$ is a synthesis transform function to reconstruct the image $x$. Minimizing Eq. (\ref{rd-in-domain}) can contribute to an optimal joint probability $p(y,z)$, which approximates the marginal distribution $p(y)$, \textit{i.e.}, Lemma \ref{lemma1} holds.  However, such marginalization approximation requires more information to be encoded~\citep{townsend2019practical}. Theoretically, we can quantify extra encoding information of marginalization approximation using the entropy $H(\cdot)$ of distribution as the rate consumption, \textit{i.e.,}
\begin{proposition}
\label{prop1}
Let $y$ and $z$ be the latent and hyper latent variables, and these variables with the asterisk be their optimal representations. In the context of in-domain image compression, if an optimal joint probability approximation of true marginal distribution can be achieved by minimizing Eq. (\ref{rd-in-domain}), the extra rate consumption of marginalization approximation is
 \begin{eqnarray}
   \Delta H^{*} = H(y,z) - H(y^{*}) =  - \log p(z|y)   \label{delat h-1},
\end{eqnarray}
\end{proposition}

\begin{proof}
\vspace{-0.2cm}
On the one hand, with joint probability and the Bayesian rule $p(y^{*})=\frac{p(y^{*}|z^{*})p(z^{*})}{p(z^{*}|y^{*})}$, we have 
\begin{eqnarray}
    H(y,z) = [-\log p(y|z) -\log p(z)],  \quad H(y^{*}) = -\log p(y^{*}|z^{*}) -\log p(z^{*}) +\log p(z^{*}|y^{*}).
\end{eqnarray}
Then, we have
\begin{eqnarray}
    \Delta H^{*}&=&[-\log p(y|z) - (-\log p(y^{*}|z^{*}))] + [-\log p(z)- (-\log p(z^{*}))] - \log p(z^{*}|y^{*}).   \label{delat h-1-2}
\end{eqnarray}
On the other hand, with Lemma \ref{lemma1}, we have
\begin{eqnarray}
    -\log p(y|z) - (-\log p(y^{*}|z^{*})) = 0, \quad s.t. \quad p(y|z)=p(y^{*}|z^{*}), \label{go to zero1 prop1}\\ -\log p(z)- (-\log p(z^{*})) = 0,\quad s.t. \quad p(z)  = p(z^{*}) \label{go to zero2 prop1}
\end{eqnarray}
Thus, 
\begin{eqnarray}
    \Delta H^{*} = - \log p(z^{*}|y^{*}) = - \log p(z|y)\label{minimal bit}
\end{eqnarray}
For in-domain image compression, Eq. (\ref{go to zero1 prop1}), Eq. (\ref{go to zero2 prop1}), and Eq. (\ref{minimal bit}) hold, as $ p(y|z) $, $p(z)$, and $p(z|y)$ are close to optimal probability representations $ p(y^{*}|z^{*}) $, $p(z^{*})$, and $p(z^{*}|y^{*})$ due to the assumption of the optimal joint
probability approximation of true marginal distribution.
\vspace{-0.2cm}
\end{proof}
Lemma \ref{lemma1} and Prop. \ref{prop1} are mainly adaptive to in-domain image compression. By this cue, we will discuss practical marginalization approximation and its impact on extra rate consumption in the context of cross-domain image compression.
\subsection{Generalize to cross-domain marginalization approximation}\label{Generalize to cross-domain marginalization approximation}
When we apply the learned image codec to cross-domain scenarios, \textit{e.g.}, assuming the source-domain image as  $x_{s}$ and the target-domain image $x_{t}$,  the \textit{i.i.d.} assumption between source and target domains violates. We can derive that Lemma \ref{lemma1} may not hold due to the following insight:
\begin{proposition}
\label{prop2}
The practical joint probability $p(y_{t},z_{t})$ on cross-domain images will deteriorate to make it not equivalent to an optimal joint probability $p(y_{t}^{*}, z_{t}^{*})$ (applicable to the cross-domain images) that also corresponds to an optimal marginal approximation distribution $p(y_{t}^{*})$, leading to  deteriorated marginalization approximation, \textit{i.e.,}
    \begin{equation}
p(y_{t}^{*}, z_{t}^{*}) \not= p(y_{t},z_{t}),  p(y_{t}^{*}) \not \approx p(y_{t},z_{t}),  p(y_{t},z_{t}) = p(y_{t}|z_{t},\theta_{h_{e}}^{s})\cdot p(z_{t}|\varphi^{s}), \label{joint prob not hold}
\end{equation}
\end{proposition}
\begin{proof}
\vspace{-0.2cm}
    The distribution shifts, \textit{i.e.,} $p(x_{t})\neq p(x_{s})$, would result in the entropy bottleneck $p(z_{t}|\varphi^{s})$ parameterized by $\varphi^{s}$ (which replaces $\varphi$ in Eq. (\ref{distribution formulate}) to emphasize its correlation with source-domain images $x_{s}$) is quite poor at specializing and encoding $z_{t}$ for cross-domain images. Also, $p(y_{t}|z_{t},\theta_{h_{e}}^{s})$ is inaccurate to encode $y_{t}$ due to source image-correlated $\theta_{h_{e}}^{s}$ and unreliable $z_{t}$. Thus, the joint probability  $p(y_{t}, z_{t})$ significantly deteriorates compared with the optimal one $p(y_{t}^{*}, z_{t}^{*})$ and further causes an unfavorable marginalization approximation $p(y_{t}^{*}) \not \approx p(y_{t},z_{t})$.
    \vspace{-0.2cm}
\end{proof} In light of Prop. \ref{prop2}, a cross-domain extension of Prop. \ref{prop1} can be derived as follows,
\begin{corollary} The extra rate consumption of cross-domain marginalization approximation $\Delta H$ will be larger than that of in-domain marginalization approximation, \textit{i.e.,},
$\Delta H > \Delta H^{*}$, due to deteriorated joint probability. \label{coro1}
\end{corollary}
\begin{proof}
\vspace{-0.4cm}
With Prop. \ref{prop2} as a condition, Eqs. (\ref{go to zero1 prop1}) and (\ref{go to zero2 prop1}) in Prop. \ref{prop1} can be further represented as follows,
\begin{eqnarray}
    -\log p(y|z) - (-\log p(y^{*}|z^{*})) > 0, \quad s.t. \quad p(y|z)\neq p(y^{*}|z^{*}), \label{go to zero1}\\ -\log p(z)- (-\log p(z^{*})) > 0,\quad s.t. \quad p(z) \neq p(z^{*}) \label{go to zero2}
\end{eqnarray}
Thus, we have 
\begin{eqnarray}
    \Delta H = [-\log p(y|z) - (-\log p(y^{*}|z^{*}))] + [-\log p(z)- (-\log p(z^{*}))] - \log p(z^{*}|y^{*}) 
    >\Delta H^{*} \label{practical bit}
\end{eqnarray} 
Eq. (\ref{practical bit}) implies more rate consumption is potentially incurred in cross-domain scenarios.
\vspace{-0.3cm}
\end{proof}
Besides, distortion error is increased as the posterior probability $p(x_{t}|y_{t},\theta_{g_{e}}^{s})$ correlated with source-domain images $x_{s}$ is suboptimal for cross-domain images $x_{t}$. To enhance the R-D performance on cross-domain images $x_{t}$, latent refinement is proposed to optimize latent variables while making model parameters unchanged. We summarize existing latent refinement methods as below.

Assume that the initial latent representations $y_{t}$ obtained by analysis transform as $y^{0}_{t}$ and the initial latent representations $z_{t}$ obtained by hyperprior transform as $z^{0}_{t}$.     
The BLR scheme (refers to Figure \ref{framework}) only updates the latent representation $y^{0}_{t}$ with $\mathrm{M}>1$ steps, for each step $m$:
\begin{equation}
    y^{m+1}_{t} = y^{m}_{t} - \epsilon \cdot \frac{\partial \mathcal{L}_{blr}}{\partial y_{t}},\quad \mathcal{L}_{blr}=-\log p(y^{m}_{t}|z^{0}_{t},\theta_{h_{e}}^{s}) + \lambda (-\log p(x_{t}|y^{m}_{t}, \theta_{g_{e}})),\label{blr}
    \end{equation}
where the R-D cost excludes  $- \log p(z)$ in Eq. (\ref{rd-in-domain}) as the latent representation $z^{0}_{t}$ is unchanged. For BLR, the gain of R-D performance is limited, as it is difficult to obtain an updated $y_{t}^{m}$ that can simultaneously lead to minimal rate cost and minimal distortion error in Eq. (\ref{blr}) due to fixed $z_{t}^{0}$. Instead, the HLR in \citep{yang2020improving} not only conducts a joint amortization gap minimization but also eliminates the discretization gap for latent representations using the following optimization step,
\begin{small}
    \begin{gather}
    y_{t}^{m+1} = y_{t}^{m} - \epsilon \cdot \frac{\partial \mathcal{L}_{hlr}}{\partial y_{t}}, z_{t}^{m+1} = z_{t}^{m} - \epsilon \cdot \frac{\partial \mathcal{L}_{hlr}}{\partial z_{t}},  \nonumber \\ \mathcal{L}_{hlr} =-\log p(y_{t}^{m}|z_{t}^{m},\theta_{h_{e}}^{s}) -  \log p(z_{t}^{m}|\varphi^{s}) + \lambda (-\log p(x_{t}|y_{t}^{m})). \label{hlr}
\end{gather}
\end{small}
By simultaneously optimizing $y_{t}$ and $z_{t}$, \citet{yang2020improving} reported significant R-D gains compared with BLR for in-domain adaptive compression. However, when we impose the HLR scheme to cross-domain adaptive compression, such gains are marginal as shown in Figure \ref{quality 0}, where the reconstruction quality of the HLR achieves significant gains compared with BLR, but the HLR consumes more bits. 
\begin{wrapfigure}{r}{0.7\textwidth}
\vspace{3pt}
    \centering
    \begin{minipage}{0.3\hsize}
    \centering
    \includegraphics[width=\hsize]{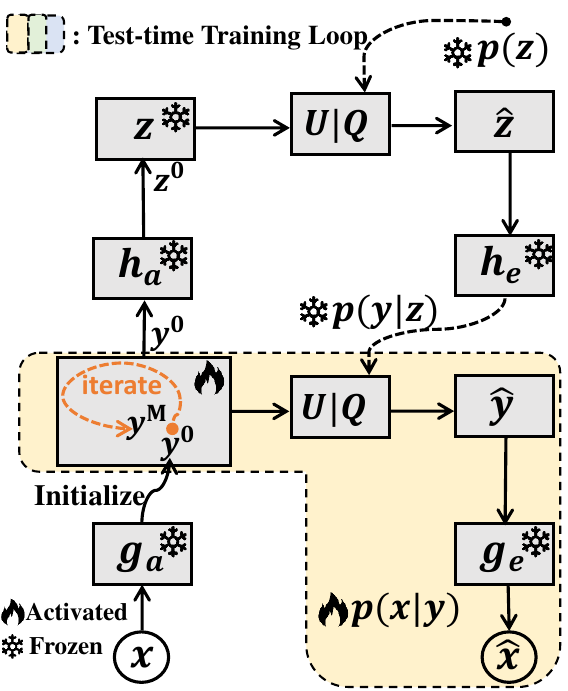}
    \vskip 4px
    \subcaption{BLR}
  \end{minipage}
  \begin{minipage}{0.3\hsize}
    \centering
    \includegraphics[width=\hsize]{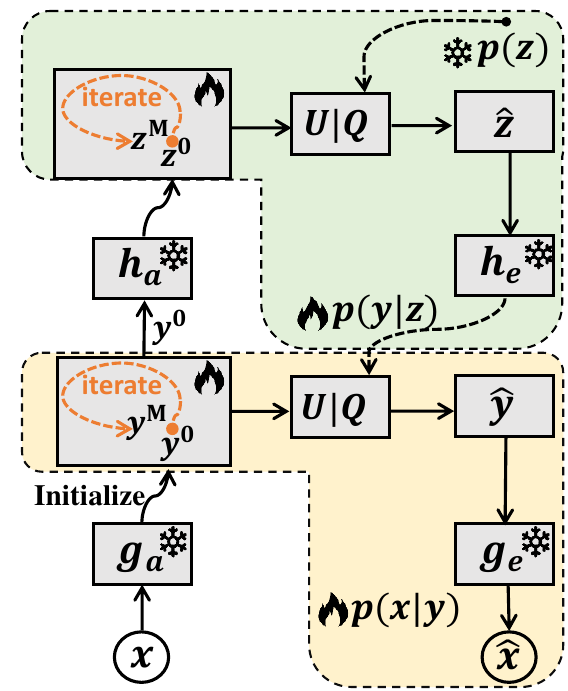}
\vskip 4px
    \subcaption{HLR}
  \end{minipage}
  \begin{minipage}{0.38\hsize}
    \centering
    \includegraphics[width=\hsize]{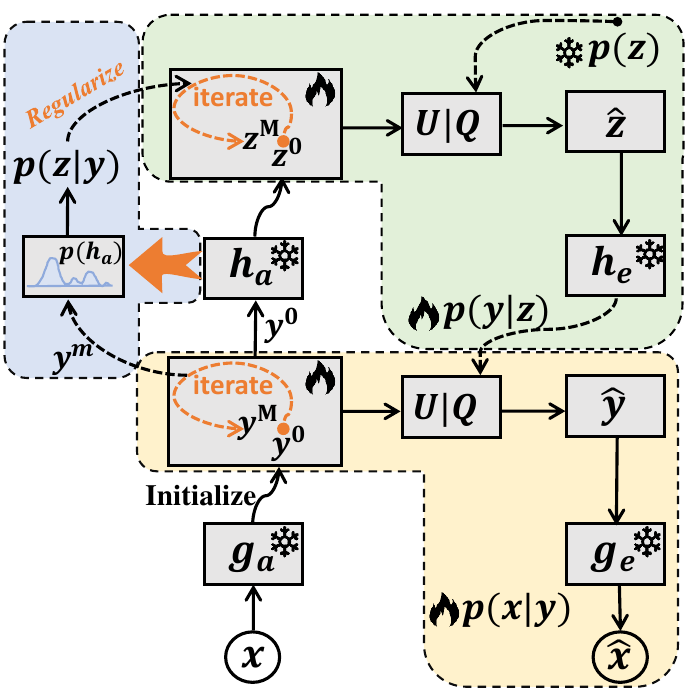}
    \vskip -3px
    \subcaption{Ours}
  \end{minipage}
\caption{Architectures of different latent refinement TTA-IC methods. $\textit{\textbf{U}}|\textit{\textbf{Q}}$ represents the quantization and entropy coding.}
\vspace{1pt}
\label{framework}
\end{wrapfigure} 

\textbf{\textit{Discussion}} \textit{ -- we provide the analyses of degradation reasons of the vanilla HLR in the cross-domain scenario using Cor. \ref{coro1} and Prop. \ref{prop2}.} First, compared with in-domain adaptive compression, the biggest trouble in the cross-domain scenario is that the static entropy bottleneck correlated with the distribution property of source-domain images cannot well model the true distribution of cross-domain latent variable $z_{t}$ as in Prop. \ref{prop2}. Second, although updating $z_{t}^{0}$ to an appropriate $z_{t}^{m}$ that can render a minimal negative log-likelihood of $p(z_{t}^{m}|\varphi^{s})$ is possible, ensuring $z_{t}^{m}$ as an appropriate Gaussian condition of $y_{t}^{m}$ is questionable for a minimal negative log-likelihood of $p(y_{t}^{m}|z_{t}^{m},\theta_{h_{e}}^{s})$. Last but not least, by recalling Porp. \ref{prop2} and Cor. \ref{coro1}, more bit consumption by the HLR implies that the finally updated joint probability $p(y_{t}^{\mathrm{M}},z_{t}^{\mathrm{M}})$ may be further away from the optimal (unknown) joint probability $p(y_{t}^{*},z_{t}^{*})$, and even the initial joint probability $p(y_{t}^{0},z_{t}^{0})$. This can be directly reflected by more rate consumption of marginal approximation by extending the result of Cor. \ref{coro1} to specific  update steps,
\begin{equation}
    \Delta H^{\mathrm{M}} > \Delta H^{0} > \Delta H^{*}.\label{delta h}
\end{equation}
To conclude, the underlying mismatch between $p({z_{t}^{m}|\varphi^{s}})$ and $p(y_{t}^{m}|z_{t}^{m},\theta_{h}^{s})$ (in the first two points) triggers deteriorated joint probability approximation of true marginal distribution (in the third point).


\subsection{Distribution regularization via Bayesian approximation}
Thus, we are interested in an advanced latent refinement method, which not only can significantly improve the reconstruction quality but also can enjoy a mild rate cost after latent refinement. Motivated by cross-domain joint probability approximation in Prop. \ref{prop2} and corresponding impact for rate consumption in Cor. \ref{coro1}, we propose to introduce a  simple yet efficient distribution regularization into the objective of vanilla HLR as follows,
\begin{equation}
    \mathcal{L}_{DR}=-\log p(y_{t}^{m}|z_{t}^{m},\theta_{h_{e}}^{s}) -\log p(z_{t}^{m}|\varphi^{s}) + \uwave{\beta(-\log p(z^{m}_{t}|y_{t}^{m}))} + \lambda (-\log p(x_{t}|y_{t}^{m})), \label{durham}
\end{equation}
where the third term is the distribution regularization has two advantages as follows. 

(i) --- If the balance coefficient $\beta$ is  1, minimizing the entropy of probability estimates  ($i.e.,$ the first two terms in Eq. (\ref{durham})) is equivalent to minimizing the distribution gap between estimated probability $p(y|z)$ or $p(z)$ and (unknown) optimal  probability $p(y^{*}|z^{*})$ or $p(z^{*})$~\citep{mackay2003information}. This coincides with the objectives of the first two terms in Eq. (\ref{practical bit}) of Cor. \ref{coro1}. Thus, $\mathcal{L}_{DR}$ can be derived as follows,
\begin{equation}
    \mathcal{L}_{DR} \propto \Delta H + \lambda (-\log p(x_{t}|y_{t}^{m})). \label{approximation durham}
\end{equation}
By recalling Eq. (\ref{practical bit}) in Cor. \ref{coro1} and Eq. (\ref{delta h}), Eq. (\ref{approximation durham}) implies that the minimization of extra rate consumption of marginalization approximation corresponds to encouraging the deteriorated joint probability approximation to approach the initial and even optimal ones. 
In other words, $\Delta H^{*}$ is the \textit{lower bound} of the first term of Eq. (\ref{approximation durham}), \textit{i.e.,} the better joint probability approximation, the closer to the lower bound. $\mathcal{L}_{DR}$ can potentially remedy the additional rate of consumption of vanilla HLR.

(ii) --- 
For the optimization process in Eq. (\ref{hlr}), 
there is no explicit constraint to ensure that $z_{t}^{m}$ can well match its posterior distribution 
under the condition of $y_{t}^{m}$, leading to deteriorated joint probability. As a practical implementation of the ideal posterior distribution, the proposed distribution regularization in Eq. (\ref{approximation durham}) can eliminate such an issue at test or compression time.

To implement the $\mathcal{L}_{DR}$, it is necessary to model the estimated posterior distribution $p(z^{m}_{t}|y_{t}^{m})$ of $z_{t}^{m}$ given $y_{t}^{m}$. However, 
for existing hyperpiror models, the practical $\hat{z}_{t}^{m}$ conditioned by $y_{t}^{m}$, \textit{i.e.,} $\hat{z}_{t}^{m}=h_{a}(y_{t}^{m};\theta_{h_{a}})$, is a deterministic output without distribution property. Although it is feasible to approximate the $\hat{z}_{t}^{m}$ as the mean $\hat{\mu}_{t}^{m}$ of a Gaussian distribution (as the assumption of posterior distribution) while constructing a variance branch from scratch~\citep{vahdat2020nvae}, the optimization of the variance model may be difficult when model parameters are fixed during latent refinement phase. To address these issues, we utilize the dropout variational inference ~\citep{gal2016dropout} as a  Bayesian approximation to existing hyperprior models for modeling the posterior distribution. Specifically, without introducing any new network, the deterministic pre-trained model $\theta_{h_{a}}$ can be formulated as a probabilistic one $p(\theta_{h_{a}})$, by treating the weights as distributions. Thus, the estimated posterior estimation can be represented as  follows,
\begin{equation}
    p(z_{t}^{m}|y_{t}^{m}) \approx \hat{p}(z_{t}^{m}|y_{t}^{m}) = \mathcal{N}(\hat{\mu}_{t}^{m}, 
    (\hat{\sigma}_{t}^{m})^{2}),
\end{equation}
where the posterior distribution is estimated and assumed as a fully factorized Gaussian distribution. The estimated mean $\hat{\mu}_{t}^{m}$ and variance $(\hat{\sigma}_{t}^{m})^{2}$ can be represented as follows,
\begin{equation}
    \hat{\mu}_{t}^{m} = \frac{1}{T}\sum_{i=1}^{T}h_{a}(y_{t}^{m};\theta_{h_{a}}^{i}), \hat{\sigma}_{t}^{m} = \frac{1}{T}\sum_{i=1}^{T}[h_{a}(y_{t}^{m};\theta_{h_{a}}^{i}) - \frac{1}{T}\sum_{i=1}^{T}h_{a}(y_{t}^{m};\theta_{h_{a}}^{i})]^{2}, \theta_{h_{a}}^{i} \sim q_{\vartheta}(\theta_{h_{a}})
\end{equation}
 For the practical dropout variational inference (DVI)~\citep{gal2016dropout}, the dropout strategy~\citep{srivastava2014dropout} is conducted to render approximate samples from the posterior distribution, which equals to use a Bernoulli variational distribution $q_{\vartheta}(\theta_{h_{a}})$, parameterized by $\vartheta$, to approximate the true model weight posterior $p(\theta_{h_{a}})$. By conducting $T$ times of Monte Carlo (MC) sampling from $q_{\vartheta}(\theta_{h_{a}})$, we can estimate the mean and the variance of the posterior distribution $p(z_{t}^{m}|y_{t}^{m})$. The dropout probability related to the Bernoulli variational distribution is set to 0.5. The computation consumption of the DVI is mild, as a single instance can parallelly sample $T$ masked weights by one inference in a batch of repeated instances. Moreover, it is flexible to degenerate to vanilla deterministic networks when the dropout probability is 1. Finally, we compute the negative log-likelihoods of the estimated posterior distribution given the current latent variable $z_{t}^{m}$.

\label{discussion}
\textit{\textbf{Discussion} -- connection with Bit-back coding (BBC).} Both BBC~\citep{townsend2019practical,ruan2021improving,ho2019compression} and our method derive from the joint probability approximation of marginal distribution. However, BBC usually specializes in \textit{in-domain} image compression to narrow the marginalization gap, \textit{i.e.,} transforming learned optimal joint probability to true marginal probability at compression time by minimizing the following objective:
\begin{equation}
    \mathcal{L}_{BBC}=-\log p(y_{t}^{m}|z_{t}^{m},\theta_{h_{e}}^{s}) -\log p(z_{t}^{m}|\varphi^{s}) - (-\log p(z^{m}_{t}|y_{t}^{m})) + \lambda (-\log p(x_{t}|y_{t}^{m})). \label{bpp}
\end{equation}
For cross-domain scenarios, such optimal joint probability does not hold due to mismatched encoding distribution as discussed in Eqs. (\ref{practical bit}) and (\ref{delta h}).
Instead, we focus on refining deteriorated joint probability to potentially optimal (even initial) one by minimizing the extra rate consumption of marginal approximation, as discussed in Eq. (\ref{approximation durham}). 
More differences refer to  \hyperlink{bbc}{Appendix}.

\section{Experiments}



\textbf{Datasets.} By following previous literature~\citep{lv2023dynamic,tsubota2023universal,shen2023dec}, we collect six different datasets with four types of image styles to comprehensively evaluate the R-D performance of different approaches on cross-domain TTA-IC tasks, including natural image (Kodak\footnote{\textcolor{Highlight}{https://r0k.us/graphics/kodak/}}), screen content image (SIQAD~\cite{yang2015perceptual}, SCID~\citep{ni2017esim}, CCT~\citep{min2017unified}), pixel-style gaming image (\cite{lv2023dynamic}' self-collected), and painting image (DomainNet~\citep{peng2019moment}) datasets. The details of the used dataset can be found in the \hyperlink{sota tta-ic}{Appendix}. Specifically, we consider the natural image dataset as in-domain evaluations, and others as cross-domain evaluations. 

\begin{figure*}[t]
  \centering
  \begin{minipage}{0.32\hsize}
    \centering
    \includegraphics[width=\hsize]{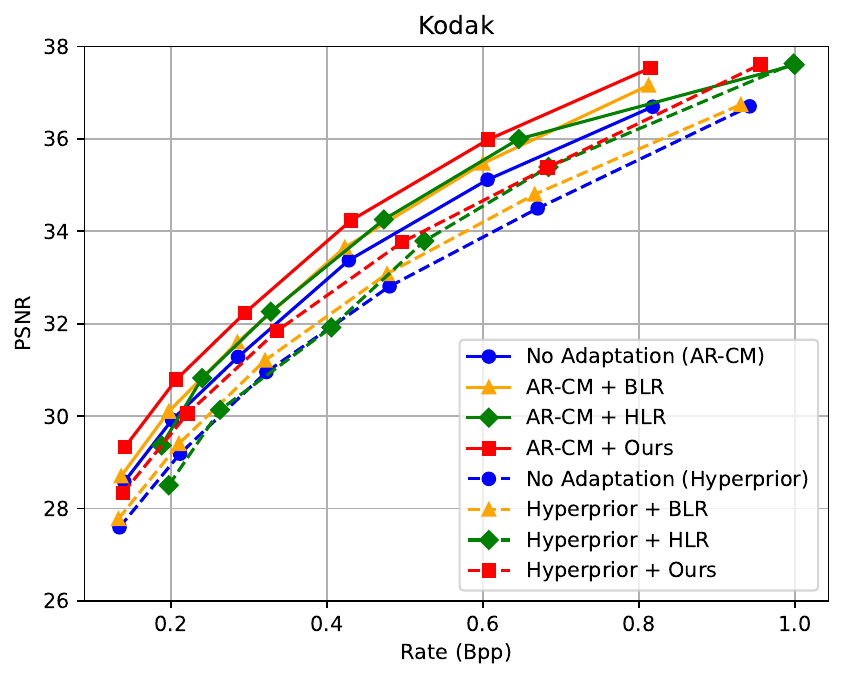}
    \vskip-5pt
    \subcaption{}
  \end{minipage}
  \begin{minipage}{0.32\hsize}
    \centering
    \includegraphics[width=\hsize]{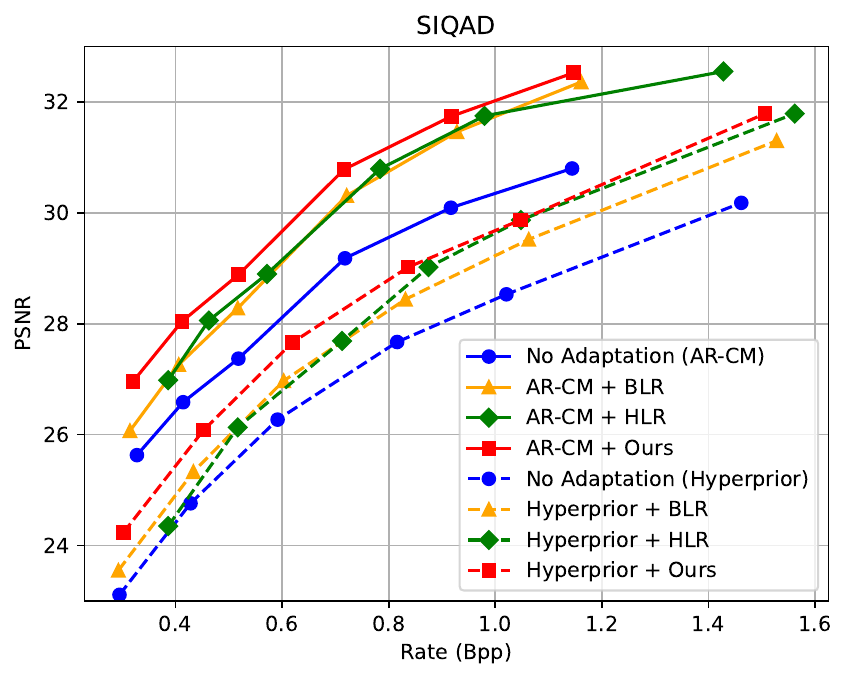}
    \vskip-5pt
    \subcaption{}
  \end{minipage}
  \begin{minipage}{0.32\hsize}
    \centering
    \includegraphics[width=\hsize]{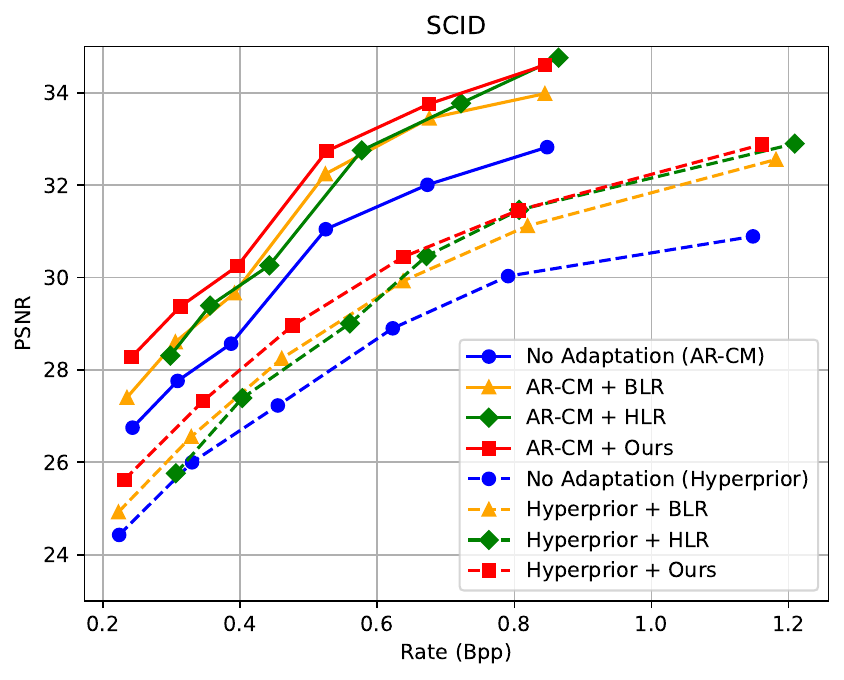}
    \vskip-5pt
    \subcaption{}
  \end{minipage}
  \begin{minipage}{0.32\hsize}
    \centering
    \includegraphics[width=\hsize]{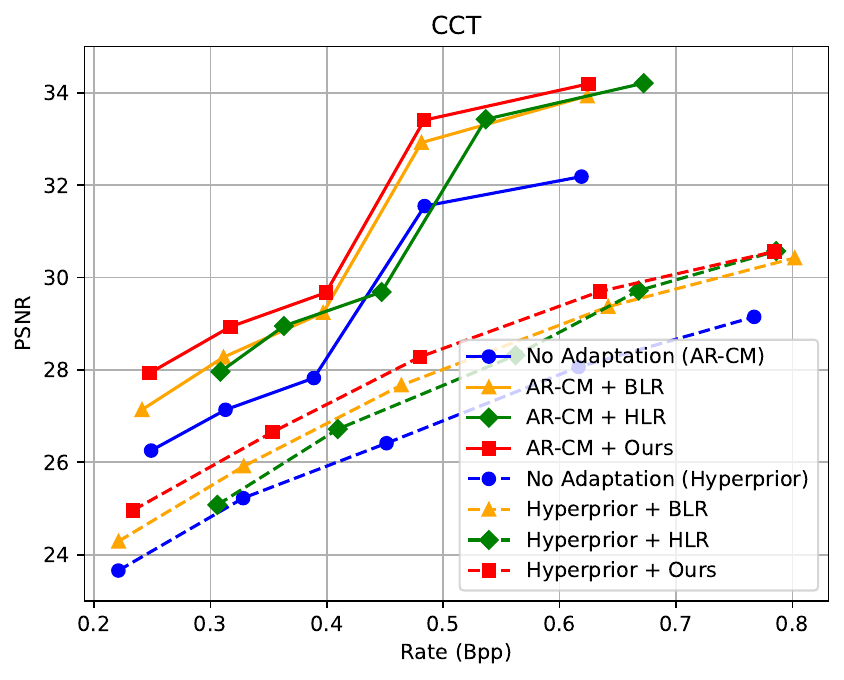}
    \vskip-5pt
    \subcaption{}
  \end{minipage}
  \begin{minipage}{0.32\hsize}
    \centering
    \includegraphics[width=\hsize]{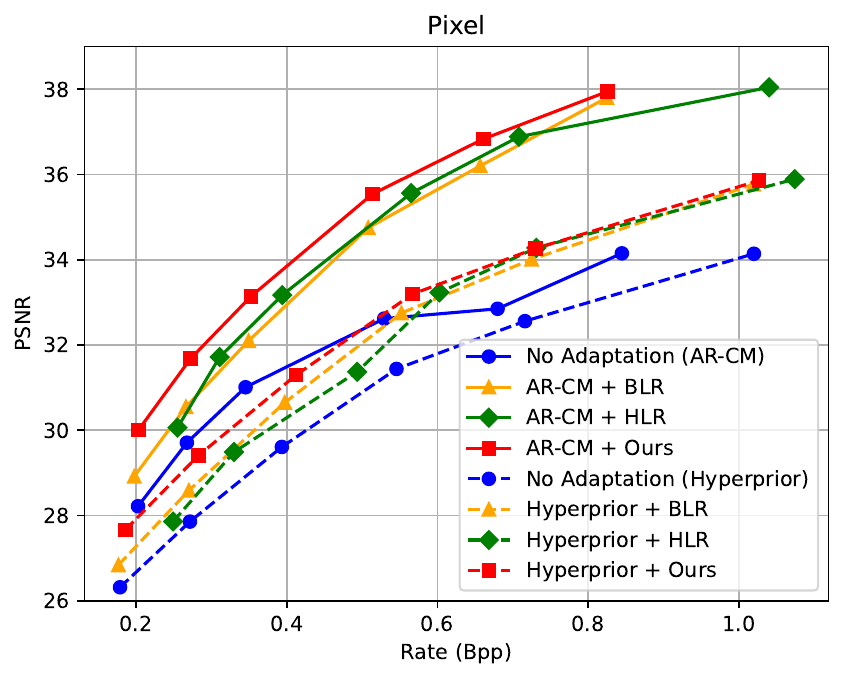}
    \vskip-5pt
    \subcaption{}
  \end{minipage}
  \begin{minipage}{0.32\hsize}
    \centering
    \includegraphics[width=\hsize]{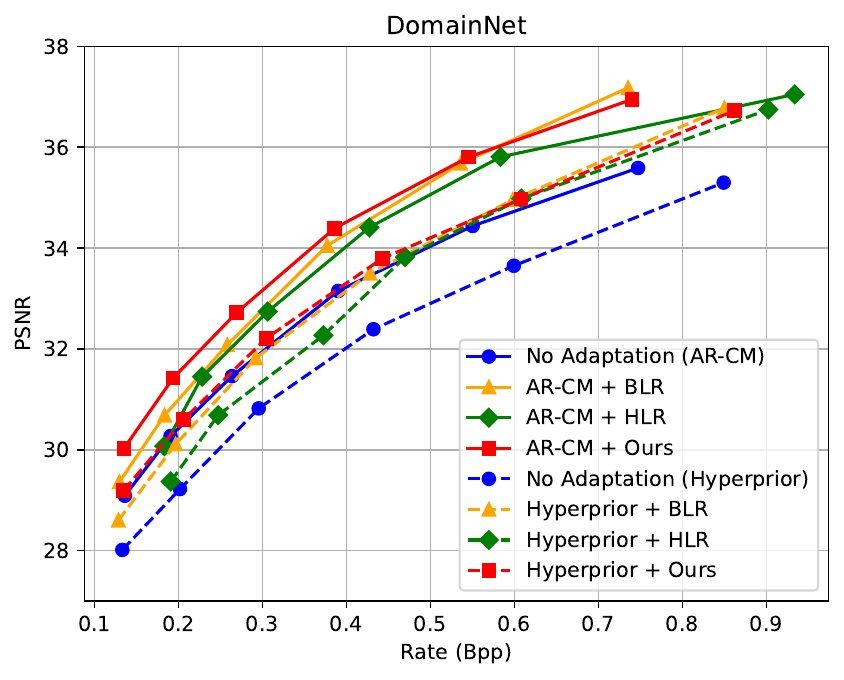}
    \vskip-5pt
    \subcaption{}
  \end{minipage}
  \caption{R-D curves on six datasets using different latent refinement methods. Two different base model architectures including AR-CM and Hyperprior are used.}
  \label{fig:rd results latent refinements}
\end{figure*}

\textbf{Implementation Details.} We use CompressAI~\citep{begaint2020compressai} to implement our proposed and baseline methods.  Two widely-used backbone models are adopted, including the base hyperprior-based entropy model proposed by \citet{balle2018variational}, namely Hyperprior, and the autoregressive context-based entropy model adopted by \citet{cheng2020learned}, namely AR-CM. Both are trained on natural images with various $\lambda$ settings. We use the pre-trained models provided by CompressAI. For TTA-IC, we use the same value of hyperparameter $\lambda=[0.0018, 0.0035, 0.0067, 0.013, 0.025, 0.048]$ for latent refinement. The Adam optimizer is utilized to update the latent variables in a learning rate of 1$\times10^{-3}$ with 2000 iterations. $T$  is empirically set to 20 for MC sampling. We discuss different implementations and hyperparameter settings (\textit{e.g.,} $\beta$) of dropout variational inference in sec. \ref{in-depth}.
\begin{table*}[!t]
\centering
  \caption{Comparison of our method with existing latent refinement approaches. The evaluation is measured in terms of average BD-rate savings (\% $\downarrow$) using the respective base models (without test-time adaptation). Smaller values indicate superior performance.}
  \label{tab:latent refinement}
  \renewcommand\arraystretch{1.2}
   \begin{adjustbox}{max width=0.8\textwidth}
  \begin{tabular}{c|c|cccccc}
    \toprule
    \multirow{2}{*}{Method} &In-domain dataset &\multicolumn{6}{c}{Cross-domain datasets}  \\ 
     & Kodak & SIQAD & SCID & CCT & Pixel & DomainNet & Average$\downarrow$
    \\ \midrule
    \rowcolor{gray!25}  AR-CM~\citep{cheng2020learned}                     
    & 0 & 0 & 0 & 0 & 0 & 0 & 0  \\
    + BLR                     
    & -6.75 & -18.75 &-18.31  &  -25.33 & -26.22  & -18.78 & -21.41  \\
    \rowcolor{gray!15} + HLR            
    & -12.73 & -21.28 &  -15.11 & -23.08 & -31.76 & -15.61 & -21.36   \\
    + Ours            
    & \textbf{-16.11}  & \textbf{-24.71} & \textbf{-22.63} & \textbf{-28.26} & \textbf{-35.89} & \textbf{-23.52} & \textbf{-27.00} \\ \hline \hline 
    \rowcolor{gray!25}  Hyperprior~\citep{balle2018variational}                     
    & 0 & 0 & 0 & 0 & 0 & 0 & 0  \\
    + BLR                     
    & -6.59 & -13.16 & -16.71 & -18.91 & -15.12 & -22.99 & -15.41  \\
    \rowcolor{gray!15}  + HLR            
    & -11.32 & -13.94 & -14.59 & -13.82 & -18.89 & -16.68 & -14.87 \\
    + Ours            
    & \textbf{-14.82}  & \textbf{-20.95} & \textbf{-23.51} & \textbf{-23.58} & \textbf{-22.22} & \textbf{-23.32} & \textbf{-20.78}  \\
  \bottomrule
\end{tabular}
\end{adjustbox}
\end{table*}
\begin{table*}
  \caption{Integration of our method with the SOTA TTA-IC approaches. The evaluation is measured in terms of average BD-rate savings (\% $\downarrow$) using the respective base models (without test-time adaptation). Smaller values indicate superior performance. $^{\ddagger}$  Matrix decomposition-based adaptor.  $^{\star}$ Entropy
efficient adapter. Stages 1 and 2 denote the latent refinement and decoder adaptation phases.}
\centering
  \label{tab: sota tta-ic}
  \renewcommand\arraystretch{1.3}
   \begin{adjustbox}{max width=0.85\textwidth}
  \begin{tabular}{c|cc|c|c|ccccc|c}
    \toprule
  \multirow{2}{*}{\textbf{Method}}   & & \multirow{2}{*}{\textbf{Type}} &\multirow{2}{*}{\textbf{ID}} &\textbf{In-domain dataset} &\multicolumn{5}{c|}{\textbf{Cross-domain datasets}} &\multirow{2}{*}{\textbf{Average}$\downarrow$} \\ 
    &  &  & & Kodak & SIQAD & SCID & CCT & Pixel & DomainNet &  
    \\ \midrule
   \multirow{5}{*}{\citet{tsubota2023universal}} &   &   WACNN~\citep{zou2022devil}  &                    
    & 0 & 0 & 0 & 0 & 0 & 0 & 0  \\ \cline{2-11}
    &  \multirow{2}{*}{\textbf{Stage 1}} &  + HLR$^{\dag}$ &  (a)                    
    & -3.50 & -13.20 &  -12.34  &  -17.15 &  -11.08 & -2.94 &  -10.02 \\
    & &  + Ours   &  (b)          
    & \textbf{-5.45} & \textbf{-14.32} & \textbf{-14.62}  &  \textbf{-18.95} & \textbf{-13.01}  & \textbf{-5.46} & \textbf{-12.01}  \\ \cline{2-11}
    & \multirow{2}{*}{\textbf{Stage 2}} & (a)  + Adaptor1$^{\ddagger}$ &       (c)               
    & -3.51 & -20.98 & -20.73 &  -24.90  &  -17.23 & -2.09 & -14.89  \\
    & &  (b) +Adaptor1$^{\ddagger}$ & (d)              
    & \textbf{-5.48} & \textbf{-22.70} & \textbf{-22.65}  & \textbf{-26.60} & \textbf{-19.37} & \textbf{-4.51} & \textbf{-16.95}   \\
    \hline 
   \multirow{5}{*}{\citet{shen2023dec}} &  &    WACNN~\citep{zou2022devil}  &                   
    & 0 & 0 & 0 & 0 & 0 & 0 & 0  \\ \cline{2-11}
    & \multirow{2}{*}{\textbf{Stage 1}} & + BLR  &  (e)                  
    & 3.83  & -10.98 & -9.96 & -13.05 & -3.13  & 4.83 &  -4.74 \\
    & &    + Ours & (f)             
    & \textbf{-5.45} & \textbf{-14.32} & \textbf{-14.62}  &  \textbf{-18.95} & \textbf{-13.01}  & \textbf{-5.46} &  \textbf{-12.01} \\ \cline{2-11}
     & \multirow{2}{*}{\textbf{Stage 2}} & (e) + Adaptor2$^{\star}$ &           (g)
    & 2.63 & -18.91 & -18.74 & -22.17 & -11.88 & 3.41 & -10.94 \\
    &  &   (f) + Adaptor2$^{\star}$  &     (h)     
    & \textbf{-6.01} & \textbf{-20.23} & \textbf{-21.02}  & \textbf{-24.23} & \textbf{-14.99} & \textbf{-5.59} &  \textbf{-15.26} \\  
  \bottomrule
\end{tabular}
\end{adjustbox}
\end{table*}

\textbf{Baselines.}
For latent refinement methods, we adopt two common baselines: (i) BLR~\citep{campos2019content} conducts content adaptation by optimizing the latent variable $y$ using Eq. (\ref{blr}). \citep{shen2023dec}
combine the BLR and decoder adaptation to achieve SOTA TTA-IC performance on screen content images. (ii) HLR~\citep{yang2020improving} concurrently optimizes the latent variable $y$ and side information $z$ by jointly minimizing the amortization and discretization gaps, as formulated in Eq. (\ref{hlr}). \citet{tsubota2023universal} and \citet{lv2023dynamic} impose the HLR and the PETL to fulfill SOTA TTA-IC performance on various target domains. For HLR, we follow \citet{yang2020improving} to use a temperature annealing schedule with defaulted hyperparameters, where $\tau_{0}=0.5$, $c_{0}=0.001$, and $t_{0}=700$. 

\textbf{Plug-and-play Validation.} We also replace different latent refinement methods of existing SOTA TTA-IC approaches~\citep{tsubota2023universal,shen2023dec} with our proposed latent refinement counterpart in a plug-and-play manner. We follow their defaulted hyperparameter settings.

\subsection{Comparison with Latent Refinement Methods}
\textbf{Rate-Distortion (R-D) Performance.} We compare our proposed method with BLR and HLR using two different pre-trained backbone models, including Hyperprior and AR-CM. Note that the model without adaptation is also involved, namely No Adaptation.  
The peak signal-to-noise ratio (PSNR) and the data rate in
bits per pixel (bpp) on different quality levels (refers to different $\lambda$ levels) are calculated to evaluate the R-D performance of different methods. Then, the R-D curves can be plotted. The results can be found in Figure \ref{fig:rd results latent refinements}. We can observe that all latent refinement methods achieve comparable performance, which may be reasonable as the Kodak dataset can be regarded as the in-domain data without domain shifts. The gains of HLR and our proposed method mainly derive from the minimization of the discretization gap compared with BLR. For out-of-domain images such as SIQAD, SCID, CCT, Pixel, and DomainNet, our proposed method outperforms other approaches with a clear margin regardless of different backbones. Especially, a better R-D performance on the low-bit conditions can be observed for AR-CM-based realizations on the Pixel dataset.

\textbf{Bjøntegaard Delta bit-rate (BD-rate) Performance}. We compute the BD-rate~\citep{Bjntegaard2001CalculationOA}, which is utilized to compare the
performance of two visual encoders at the same bit rate. A lower
BD rate demonstrates better performance. The baseline model without adaptation is the anchor model for BD-rate calculation. The results are shown in Table \ref{tab:latent refinement}. Some observations can be summarized. First, the HLR and our proposed method outperform the BLR with a clear margin for the in-domain Kodak dataset, which coincides with \citet{yang2020improving}' results. However, on the cross-domain tasks, the HLR suffers from significant degradations compared with BLR, where HLR performs below the HLR on three out of five tasks, \textit{e.g.,} SCID, CCT, and DomainNet. This is reasonable as the static entropy bottleneck correlated with the distribution property of natural images cannot model the true distribution of latent variable $z$ well. Finally, our method can achieve obvious gains on all cross-domain tasks, when the proposed distribution regularization is plugged into vanilla HLR.
\begin{figure*}[!t]
    \centering
    \begin{minipage}{0.24\hsize}
    \centering
    \includegraphics[width=\hsize]{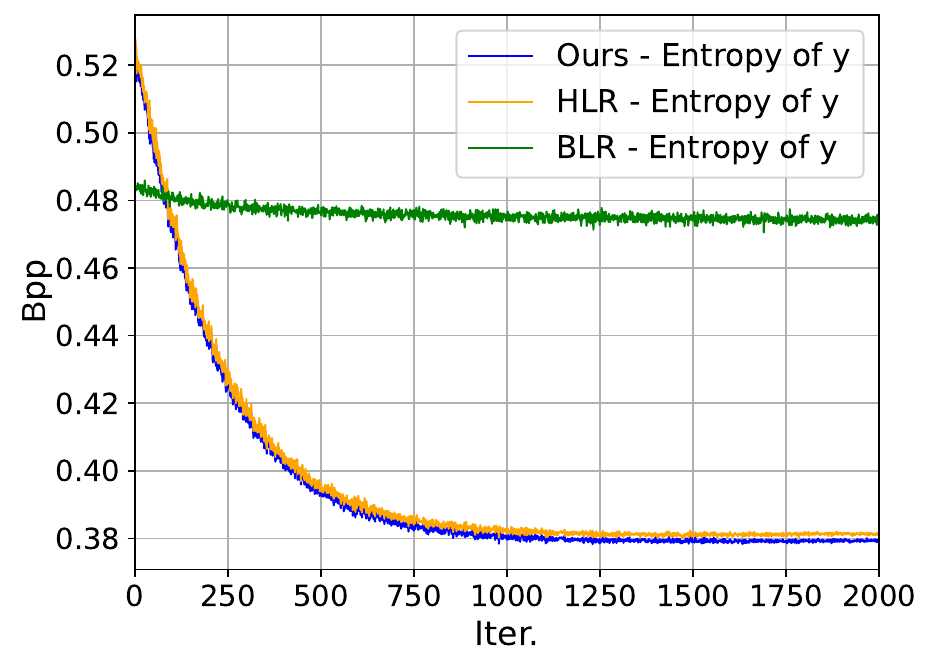}
    \vskip-5pt
    \subcaption{}
  \end{minipage}
  \begin{minipage}{0.24\hsize}
    \centering
    \includegraphics[width=\hsize]{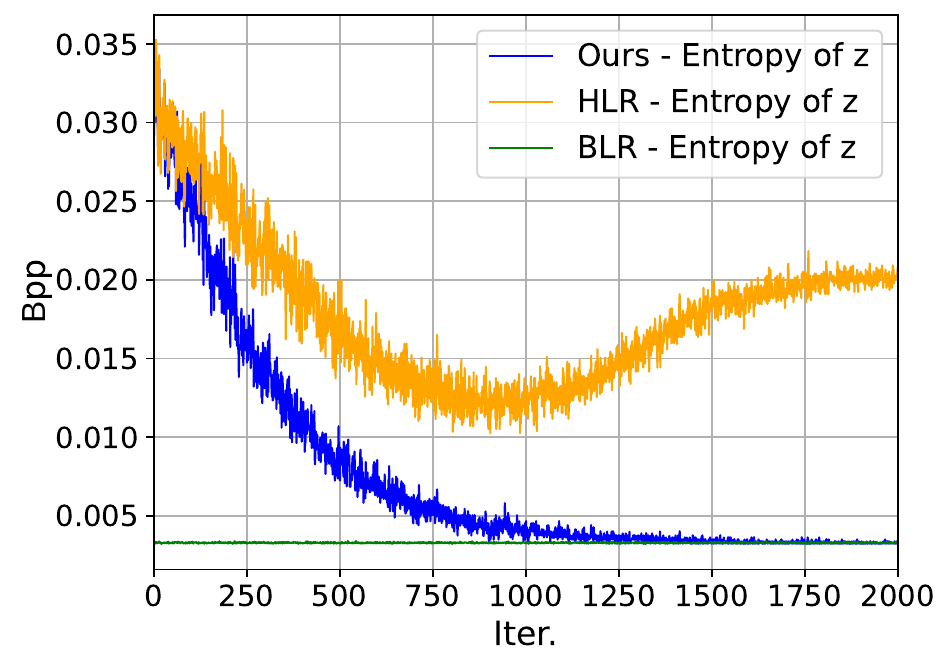}
    \vskip-5pt
    \subcaption{}
  \end{minipage}
  \begin{minipage}{0.24\hsize}
    \centering
    \includegraphics[width=\hsize]{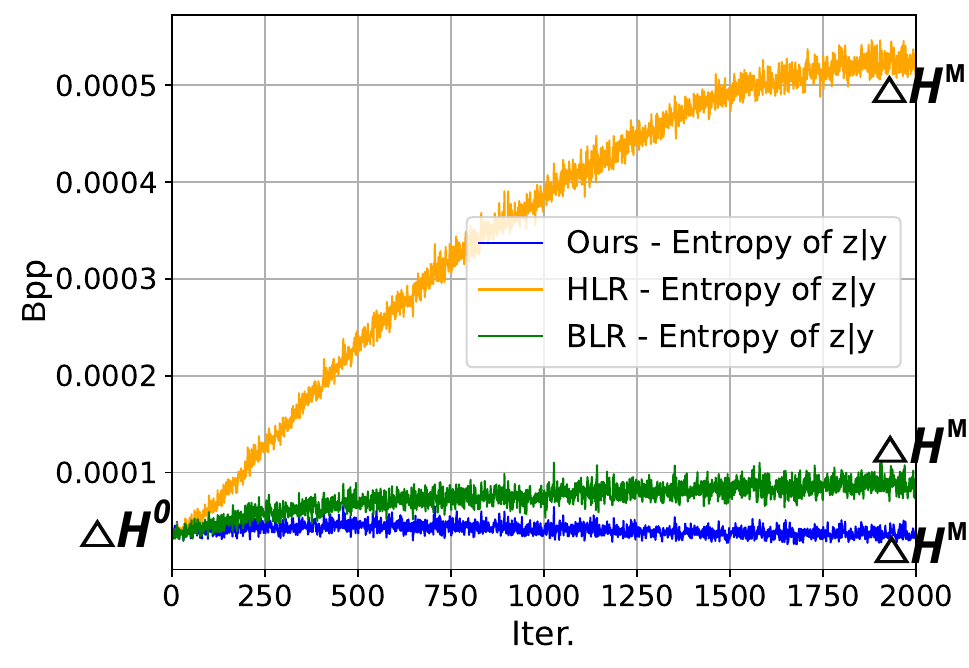}
    \vskip-5pt
    \subcaption{}
  \end{minipage}
  \begin{minipage}{0.24\hsize}
    \centering
    \includegraphics[width=\hsize]{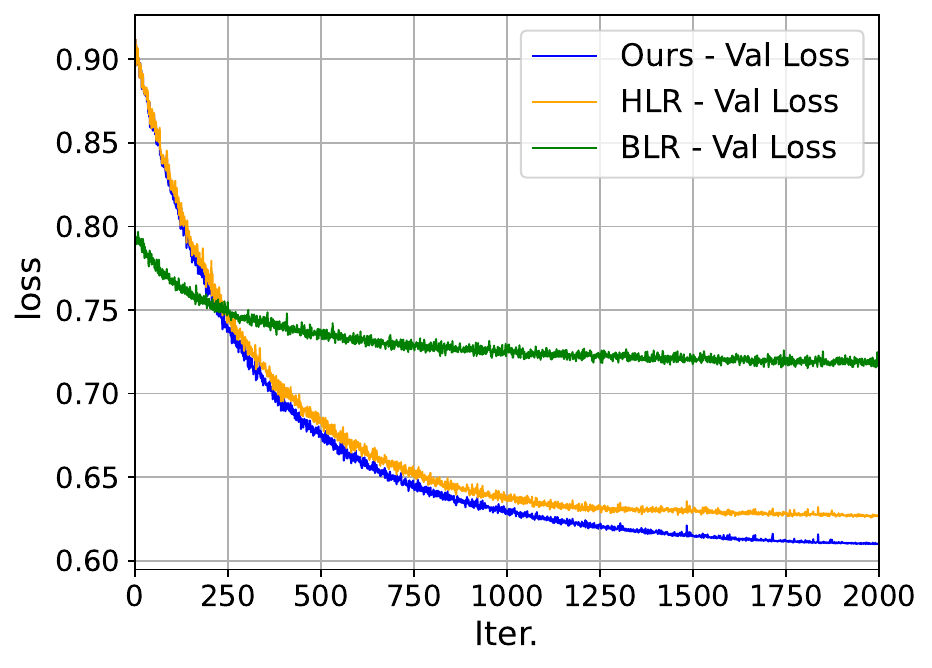}
    \vskip-5pt
    \subcaption{}
  \end{minipage}
\caption{Entropy curves of different probabilities with the iteration $t$, including (a) $-\log p(y^{m})$ (b)$-\log p(z^{m})$ (c) $-\log p(z^{m}|y^{m})$ (d) validation loss as calculated in Eq. (\ref{rd-in-domain}).}
\label{entropy}
\end{figure*}
\begin{figure}[!t]
\centering
\sbox{\measurebox}{%
  \begin{minipage}[b]{.48\textwidth}
  \subfloat
  \centering
    {\label{fig:figA}\includegraphics[width=0.99\textwidth]{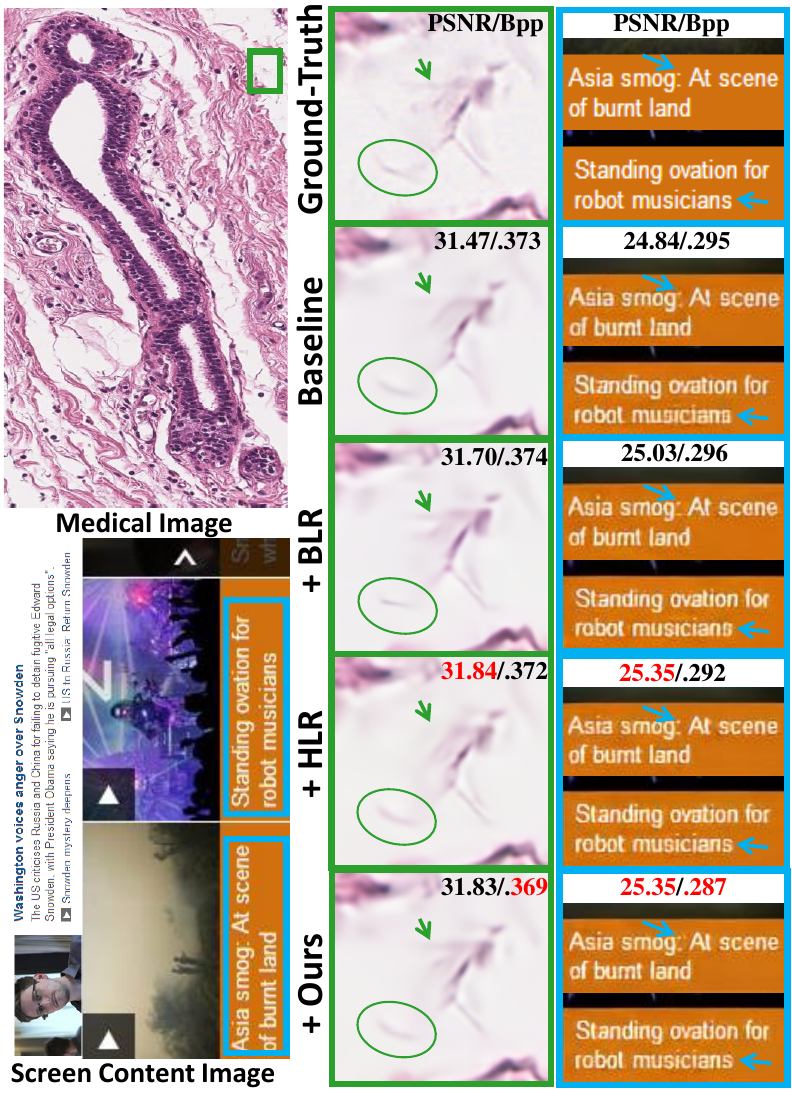}}
    \captionof{figure}{Qualitative results for different domains.}
    \label{visual}
  \end{minipage}}
\usebox{\measurebox}\quad
\begin{minipage}[b][\ht\measurebox][s]{.48\textwidth}
\centering
\subfloat
{\label{fig:figB}\includegraphics[width=0.65\textwidth]{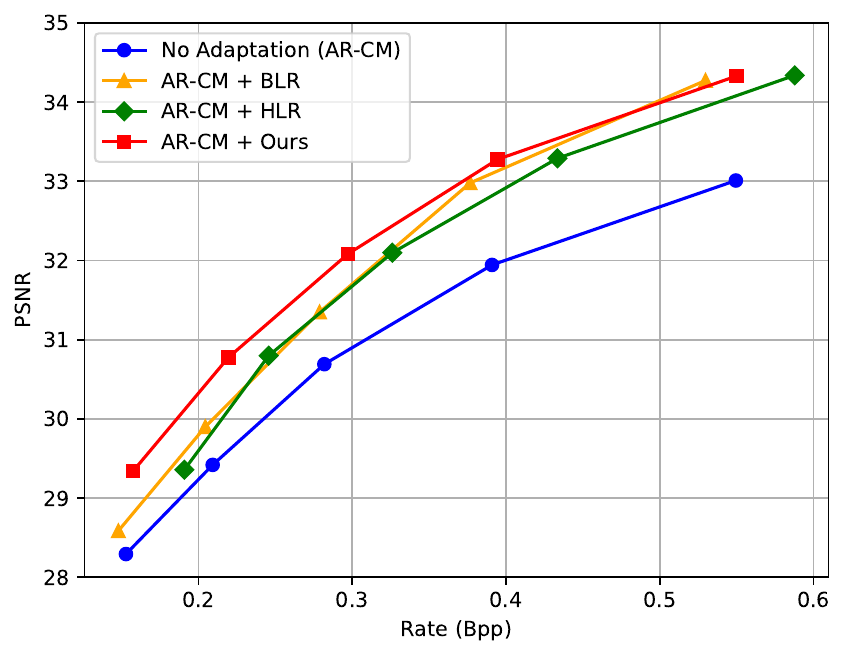}}
\captionof{figure}{R-D performance on challenging medical images, where pathological breast cancer images are used for cross-domain compression~\citep{niazi2019pathological}. }
		\label{loss}
\hfill 
\vspace{-0.3cm}
\subfloat
  \centering
		\renewcommand{\arraystretch}{1.}
		\captionof{table}{Ablation study of proposed distribution regularization $\mathcal{L}_{DR}$ in terms of BD-rate (\% $\downarrow$) on SIQAD. DL denotes the number of the dropout layer. $^{\dag}$: assume $\sigma=1$ without Bayesian approximation. $^{\ddagger}$: minimizing Eq. (\ref{bpp}).} 
  \vspace{0.1cm}
		\label{ablation}
        \setlength{\tabcolsep}{3pt}
		\begin{adjustbox}{width=0.95\textwidth}
		\begin{tabular}{ccccc}
\hline
\multicolumn{1}{c|}{Baseline (HLR)}                                                      & \multicolumn{4}{c}{0: As an anchor}     \\ \cline{1-5} 
\multicolumn{1}{c|}{\multirow{4}{*}{+ proposed $\mathcal{L}_{DR}$}} &  
\multicolumn{1}{c|}{\diagbox{DL}{$\beta$}}   & 0.1 & 1 & 10  \\  \cline{2-5}
\multicolumn{1}{c|}{}                                               & \multicolumn{1}{c|}{1}   &  -10.01      &  -10.02     &    -7.07   \\
\multicolumn{1}{c|}{}  & \multicolumn{1}{c|}{2}    &   -10.03    &   -9.62    &  -5.25   \\
\multicolumn{1}{c|}{}                                               & \multicolumn{1}{c|}{3}   & \cellcolor{green!20}{\textbf{-10.06}}       &   -9.61    &  -4.95     \\ \hline\multicolumn{1}{c|}{+ Deterministic $\mathcal{L}_{DR}^{\dag}$ }                                             & \multicolumn{4}{c}{-3.54} \\ 
\hline\multicolumn{1}{c|}{+ BBC \citep{yang2020improving}$^{\ddagger}$}                                               & \multicolumn{4}{c}{0.15} \\ \hline
\end{tabular}
        \end{adjustbox}

\end{minipage}
\vspace{-1.5cm}
\end{figure}
\subsection{Comparison with SOTA TTA-IC -- Plugging into Distribution Regularization}

Moreover, we replace the existing
latent refinement methods used by SOTA TTA-IC approaches with our proposed method. The full R-D curves are in the \hyperlink{sota tta-ic}{Appendix}. We present the BD-rate in Table \ref{tab: sota tta-ic}. 
It is obvious that our proposed method can improve the performance of these TTA-IC methods on all tasks. For example, when replacing vanilla HLR adopted by \citep{tsubota2023universal} with our proposed distribution regularization-based alternative, the BD-rate score improves from \textbf{-2.94} to \textbf{-5.46} on the DomainNet dataset for the latent refinement stage. Moreover, such gains are roughly maintained at a similar level when adopted by the decoder adaptation. Due to the benefit of hybrid latent refinement, these gains are significantly enlarged for \citet{shen2023dec}. Thus, our proposed method enjoys plug-and-play property, benefiting existing TTA-IC methods.
\begin{figure}[!h]
    \centering
    \caption{Correlation between adaptation performance
and adaptation time (using a single NVIDIA GeForce 3090 GPU) on SIQAD.}
\vspace{-0.15cm}
\renewcommand{\arraystretch}{1.}
	\begin{adjustbox}{width=0.65\textwidth}
		\begin{tabular}{c|llllll}
\toprule
                         & \textbf{Steps} & \multicolumn{1}{c}{\textbf{1}}   & \multicolumn{1}{c}{\textbf{500}} & \multicolumn{1}{c}{\textbf{1000}} & \multicolumn{1}{c}{\textbf{1500}} & \multicolumn{1}{c}{\textbf{2000}} \\ \hline
\multicolumn{1}{c|}{\multirow{3}{*}{\begin{tabular}[c]{@{}c@{}}\textbf{Runtime}\\ \textbf{(Avg. Sec./Img.)}\end{tabular}}} & BLR   &  0.07s &  13.96s   &   27.32s  &    41.52s  &   56.28s   \\
& HLR   &  0.07s &   14.86s  &   28.52s  &   42.51s   &   57.58s   \\
                         &   Ours  & 0.08s  &   15.48s  &  29.25s   & 44.85s     &  59.70s    \\ \midrule
\multirow{3}{*}{\textbf{BD-rate(\%)$\downarrow$}} & BLR   &  \cellcolor{green!20}{\textbf{-1.01}} &  -19.23   &  -20.57   &     -21.04 &    -21.27  \\
& HLR  & -0.92  & -12.75    &  -14.73   &   -19.32   & -23.68     \\ 
                         & Ours  &  -0.95  &   \cellcolor{green!20}{\textbf{-24.33}}  & \cellcolor{green!20}{\textbf{-29.88}}     &   \cellcolor{green!20}{\textbf{-33.15}}   &   \cellcolor{green!20}{\textbf{-34.20}}   \\ \bottomrule
\end{tabular}
        \end{adjustbox}
        \vspace{0.1cm}
        \label{efficiency}
\end{figure} 
\subsection{In-depth Analyses of Our Proposed Method}
\label{in-depth}
We gain an in-depth understanding of the proposed distribution regularization from four views. 

\textbf{Why does the proposed distribution regularization improve the R-D performance?} \textbf{(i)} From theoretical analyses in section \ref{Generalize to cross-domain marginalization approximation}, the domain shifts and the underlying mismatch between $p({z^{m}|\varphi^{s}})$ and $p(y_{t}^{m}|z_{t}^{m},\theta_{h_{e}}^{s})$ may trigger the deteriorated joint probability approximation $p(y_{t}^{M},z_{t}^{M})$ of true marginal distribution $p(y_{t}^{*})$ for vanilla HLR. By empirical experiments in Figure \ref{entropy}, we observe that it is indeed difficult to jointly optimize latent variable $y$ and side information $z$ for vanilla HLR as shown in Figs. \hyperlink{\ref{entropy}}{4(a)} and \hyperlink{\ref{entropy}}{4(b)}, where $-\log p(z^{m})$ converges first and degrades with the iterations while $-\log p(y^{m})$ converges well. In contrast, due to the lack of optimization for slide information $z$, there is limited convergence of $-\log p(y^{m})$ for the BLR while $-\log p(z^{m})$ is unchanged. Thus, as shown in Figure
\hyperlink{\ref{entropy}}{4(c)}, we can find an obvious degradation for vanilla HLR, \textit{i.e.,} $\Delta H^{0} < \Delta H^{\mathrm{M}}$,  reflecting more rate consumption of marginal approximation. This implies that the finally updated joint probability $p(y^{\mathrm{M}},z^{\mathrm{M}})$ is further away from the optimal joint probability $p(y^{*},z^{*})$, and even the initial joint probability $p(y^{0},z^{0})$. \textbf{(ii)} By introducing the proposed distribution regularization, as shown in Figure \hyperlink{\ref{entropy}}{4(c)}, the degradation of joint probability approximation is alleviated well (blue curve) compared with vanilla HLR. Even, our method can achieve good convergence, \textit{i.e.,} $\Delta H^{0} > \Delta H^{\mathrm{M}}$, which is reasonable as we directly minimize $\Delta H$ as shown in Eq. (\ref{approximation durham}). \\
\textit{In short, the proposed distribution regularization encourages the deteriorated joint probability approximation to approach the initial and even unknown optimal ones, leading to good convergence of rate consumption and a better R-D performance in Figure \ref{entropy}.}\\
\textbf{How about the effectiveness of Bayesian approximation for distribution regularization?} To validate the effectiveness of Bayesian approximation for distribution regularization, we first ablate the number of dropout variational inference layers. Intuitively, more dropout layers facilitate a more accurate estimate of posterior distribution, due to more powerful distribution representations. As shown in Table \ref{ablation}, there are marginal gains with the improvement of dropout layers. Moreover, over-regularization, \textit{i.e.,} $\beta=10$, will result in obviously negative effects. Second, we try a deterministic version of our distribution regularization by removing the Bayesian approximation, \textit{i.e.,} assuming $\sigma = 1$ without the dropout layer. As we can see, the performance of deterministic $\mathcal{L}_{DR}$ performs significantly below that of the Bayesian approximation-based one, which is reasonable as the over- or below-estimated variance of the posterior distribution is inaccurate compared with the practical dropout variational inference in Monte Carlo sampling. Finally, we observe negligible gains of the BBC on cross-domain image compression, which coincides with our discussions in section \ref{discussion}.\\
\textbf{How about the adaptation efficiency compared with baseline methods?} As illustrated in Table \ref{efficiency}, we can observe that our proposed method has obvious performance gains in different adaptation steps. In contrast, the additional adaptation time taken by ours is mild in the same adaptation step, which is reasonable as a single instance can parallelly sample $T$ masked weights by one inference in a batch of repeated instances for Bayesian approximation. Moreover, using a more computationally efficient GPU can further reduce the adaptation time in the future.\\
\textbf{What is the ability to generalize for larger distribution discrepancy?}
We explore whether our proposed method can be scalable to more challenging medical images, where the image distribution of medical images is quite different from natural images. The qualitative and quantitative results in Figures. \ref{visual} and \ref{loss} show our method can improve R-D performance with better texture preservation. 
\vspace{-0.2cm}
\section{Conclusion}
\vspace{-0.2cm}
We have approached an advanced latent refinement method by tailoring the vanilla HLR method designed for \textit{in-domain} inference improvement to \textit{cross-domain} cases. Specifically, we have provided a theoretical analysis to uncover the degradation reason of rate cost for the vanilla HLR,  \textit{i.e.,} underlying mismatch between refined Gaussian conditional and hyperprior distributions may trigger the deteriorated joint probability approximation of marginal distribution, leading to increased rate consumption. Then, we introduce a Bayesian approximation-endowed \textit{distribution regularization} to encourage learning better joint probability approximation in a plug-and-play manner. Extensive experiments demonstrate that our proposed method can achieve promising performance.


\bibliography{iclr2025_conference}
\bibliographystyle{iclr2025_conference}
\newpage
\appendix
\section{Details of adopted datasets}
By following previous literature~\citep{lv2023dynamic,tsubota2023universal,shen2023dec}, we collect six different datasets with four types of image styles to comprehensively evaluate the R-D performance of different approaches on cross-domain TTA-IC tasks, including natural image (Kodak\footnote{\textcolor{Highlight}{https://r0k.us/graphics/kodak/}}), screen content image (SIQAD~\cite{yang2015perceptual}, SCID~\citep{ni2017esim}, CCT~\citep{min2017unified}), pixel-style gaming image (\cite{lv2023dynamic}' self-collected), and painting image (DomainNet~\citep{peng2019moment}) datasets. The details of the used dataset can be found in the Table \ref{dataset}. Specifically, we consider the natural image dataset as in-domain evaluations, and others as cross-domain evaluations. 
\begin{table}[!h]
    \centering
    \caption{The datasets for evaluation. The symbols $*$ and $\dag$ denote in- and cross-domain datasets, respectively.}
 \begin{adjustbox}{max width=0.6\columnwidth}
	\begin{tabular}{cccc}
\hline
\textbf{Dataset}         & \textbf{Description} &  \textbf{\# Num.}  & \textbf{Avg. Resolution}\\ \hline \hline
Kodak$^{*}$   & Natural            & 24 &  576$\times$704   \\
SIQAD$^{\dag}$          & Screen Content &  24    &  685$\times$739  \\
SCID$^{\dag}$ & Screen Content              &  40   &     720$\times$1080 \\
CCT$^{\dag}$ &   Screen Content           & 24     & 915$\times$1627 \\
Pixel$^{\dag}$           & Gaming             & 25   &  746$\times$850\\
DomainNet$^{\dag}$         & Painting             &  25  &   492$\times$640   \\ \hline   
\end{tabular}
\label{dataset}
\end{adjustbox}
\vspace{0.2cm}
\end{table} 
\section{Comparison with SOTA TTA-IC Methods}
\label{sota tta-ic}
We provide the full R-D curves in Figure \ref{fig: appendix rd results latent refinements} when we compared our proposed method with SOTA TTA-IC methods. Note that our proposed method is based on \citep{tsubota2023universal} with our proposed distribution regularization. We can observe that all latent refinement methods achieve comparable performance, which may be reasonable as the Kodak dataset can be regarded as the in-domain data without domain shifts. The gains of HLR and our proposed method mainly derive from the minimization of the discretization gap compared with BLR. For out-of-domain images such as SIQAD, SCID, CCT, Pixel, and DomainNet, our proposed method outperforms other approaches with a clear margin regardless of different backbones. Especially, a better R-D performance on the low-bit conditions can be observed for AR-CM-based realizations on the Pixel dataset.
\begin{figure*}[!t]
  \centering
  \begin{minipage}{0.32\hsize}
    \centering
    \includegraphics[width=\hsize]{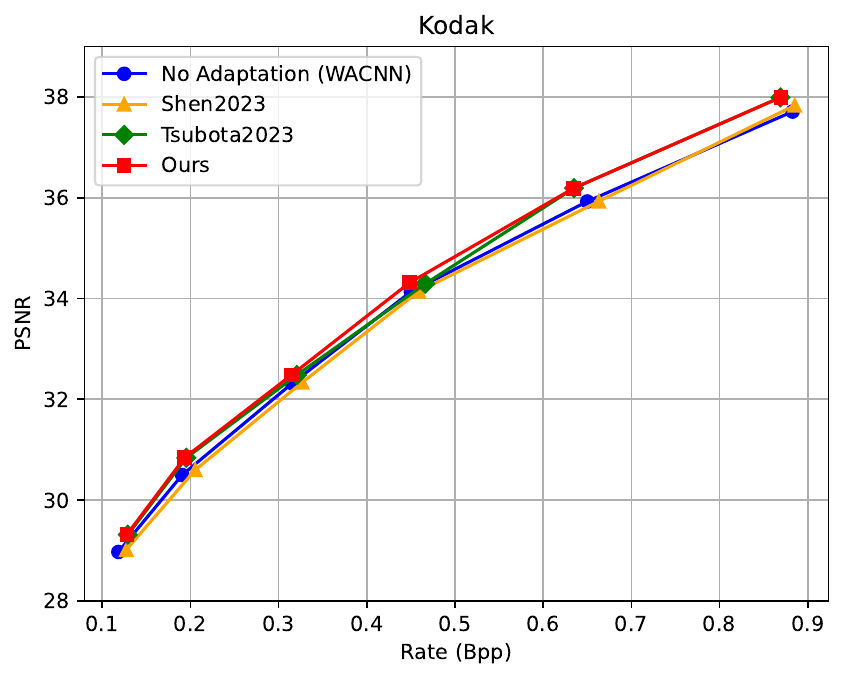}
    \vskip-5pt
    \subcaption{}
  \end{minipage}
  \begin{minipage}{0.32\hsize}
    \centering
    \includegraphics[width=\hsize]{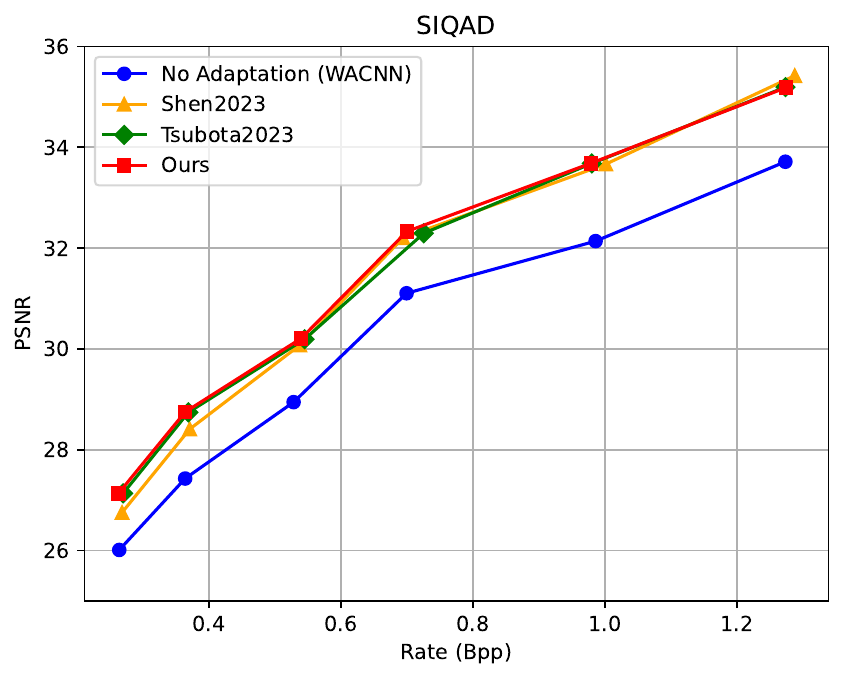}
    \vskip-5pt
    \subcaption{}
  \end{minipage}
  \begin{minipage}{0.32\hsize}
    \centering
    \includegraphics[width=\hsize]{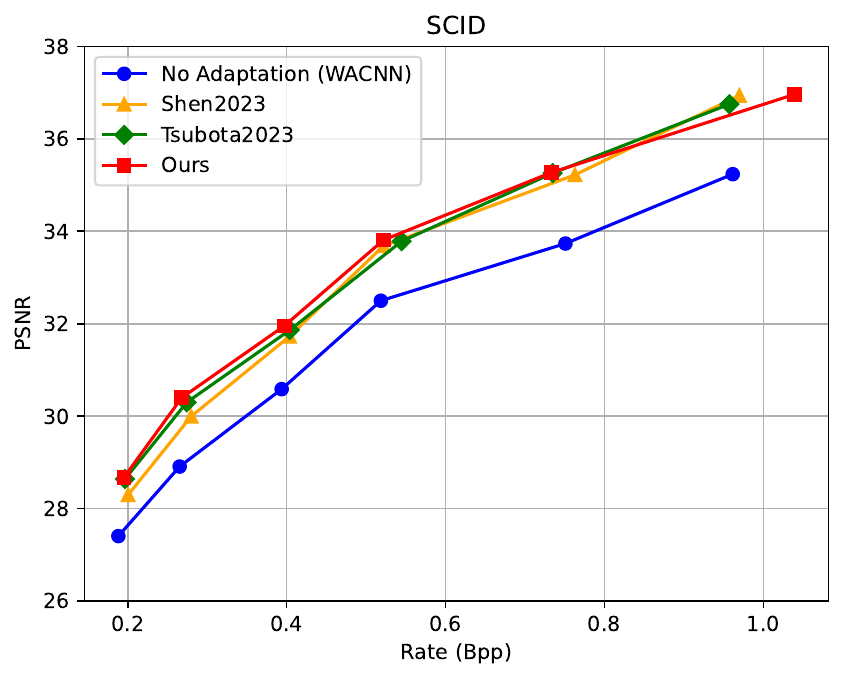}
    \vskip-5pt
    \subcaption{}
  \end{minipage}
  \begin{minipage}{0.32\hsize}
    \centering
    \includegraphics[width=\hsize]{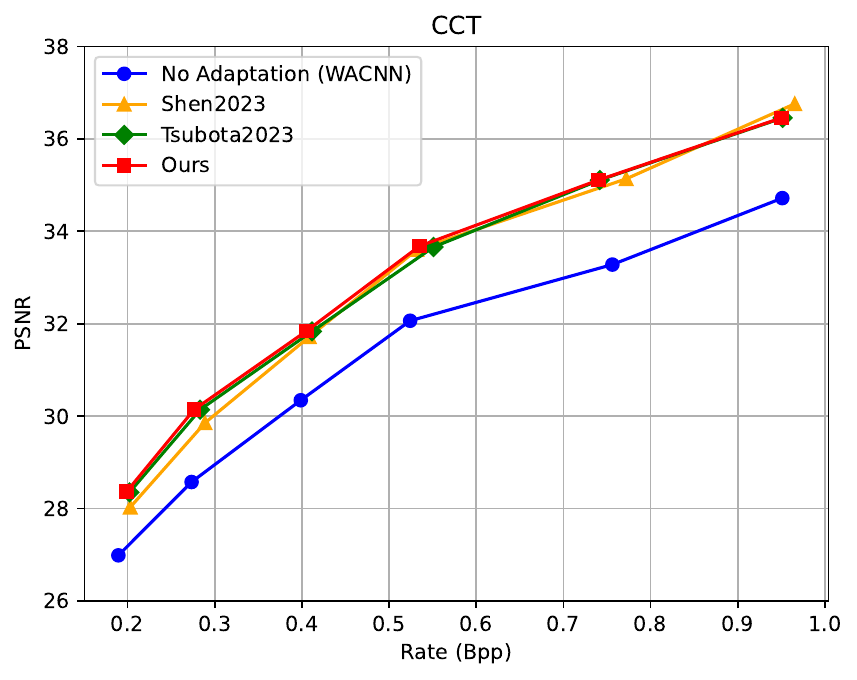}
    \vskip-5pt
    \subcaption{}
  \end{minipage}
  \begin{minipage}{0.32\hsize}
    \centering
    \includegraphics[width=\hsize]{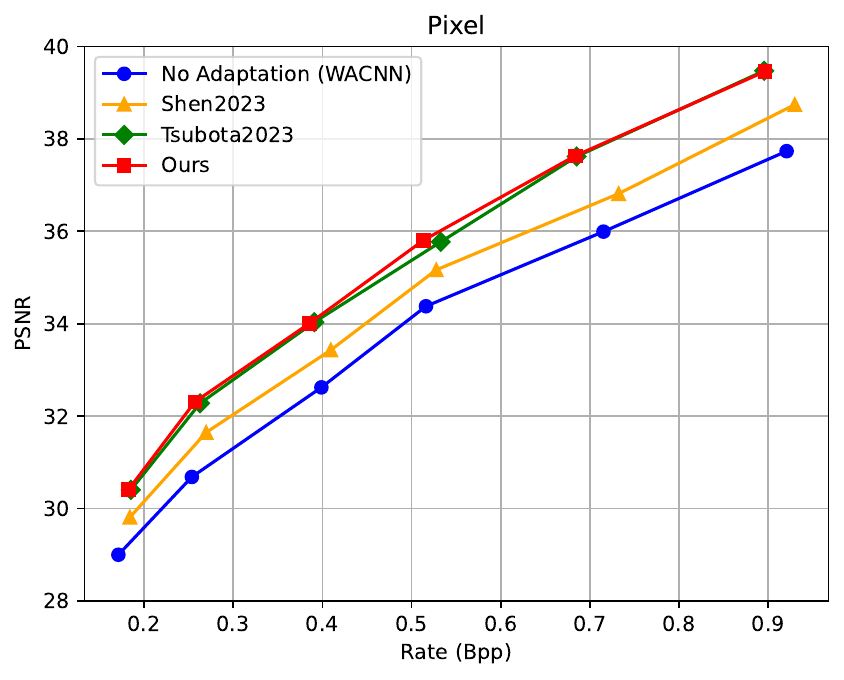}
    \vskip-5pt
    \subcaption{}
  \end{minipage}
  \begin{minipage}{0.32\hsize}
    \centering
    \includegraphics[width=\hsize]{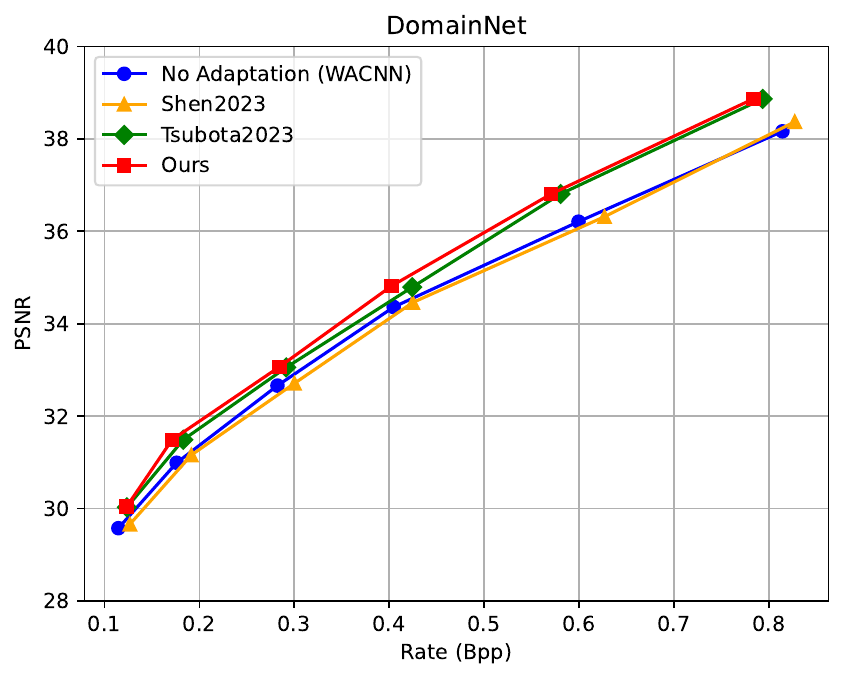}
    \vskip-5pt
    \subcaption{}
  \end{minipage}
  \caption{R-D curves of test-time adaptation of different latent refinement approaches on six datasets using different latent refinement methods. Two different base model architectures including AR-CM and Hyperprior are used.}
  \label{fig: appendix rd results latent refinements}
\end{figure*}
\section{Connection with bit-back coding}
\label{bbc}
1) \textit{In-domain} \textit{v.s.} \textit{Cross-domain}: Both BBC~\citep{townsend2019practical,ruan2021improving,ho2019compression} and our proposed method derive from the joint probability approximation of marginal distribution. However, BBC usually specializes in \textit{in-domain} image compression to narrow the marginalization gap, \textit{i.e.,} transforming learned optimal joint probability to true marginal probability at compression time. 
Instead, in the context of \textit{cross-domain} image compression, such optimal joint probability does not hold due to mismatched encoding distribution (\textit{e.g.,} entropy bottleneck), as discussed in Eqs. (\ref{practical bit}) and (\ref{delta h}). 

2) \textit{Minimizing posterior probability v.s. Maximizing posterior probability}: BBC minimizes the second term of the first line of Eq. (\ref{delat h-1}) as the true entropy of the marginal distribution, \textit{i.e.,}
\begin{equation}
    \mathcal{L}_{BBC}=-\log p(y_{t}^{m}|z_{t}^{m},\theta_{h_{e}}^{s}) -\log p(z_{t}^{m}|\varphi^{s}) - (-\log p(z^{m}_{t}|y_{t}^{m})) + \lambda (-\log p(x_{t}|y_{t}^{m})), \label{bpp2}
\end{equation}
where the third term corresponds to minimizing the posterior probability based on the in-domain optimal joint probability assumption of BBC. In contrast, our proposed method focuses on refining deteriorated joint probability to optimal one by minimizing the extra rate consumption of marginal approximation, as discussed in Eq. (\ref{approximation durham}). 

3) \textit{Performance and implementation differences}: Our experiments in Table \ref{ablation} observe negligible gains of BBC on cross-domain image compression. More importantly, in order to acquire distribution property, our proposed distribution regularization relies on more flexible variational Bayesian inference rather than introducing additional networks like BBC~\citep{yang2020improving}.
\end{document}